%% file: main.tex
\documentclass[letterpaper, 10 pt, conference]{ieeeconf}  
\IEEEoverridecommandlockouts

\let\proof\relax
\let\endproof\relax
\usepackage{caption}
\usepackage{amsthm}
\usepackage[utf8]{inputenc}
\usepackage{amsmath,amssymb}
\usepackage{booktabs}
\usepackage{graphicx}
\usepackage[font=small]{subcaption}
\usepackage{numprint}
\usepackage{comment}
\usepackage[linesnumbered,ruled,noend]{algorithm2e}
\usepackage{eso-pic}
\usepackage[american]{babel}
\usepackage{svg}
\usepackage{xspace}

\usepackage{adjustbox}
\usepackage{tikz}
\usetikzlibrary{automata,positioning,shapes,arrows}

\usepackage[normalem]{ulem} 

\usepackage{tabularx,siunitx}
\usepackage{cite}
\usepackage{hyperref}
\usepackage{multirow}

\newcommand{\automaton}{\mathcal{A}}

\usepackage{xcolor}

\SetCommentSty{mycommfont}

\usepackage{adjustbox}
\usepackage{tikz}
\usetikzlibrary{automata,positioning,shapes,arrows}
\usepackage[normalem]{ulem} 
\usepackage{tabularx,siunitx}
\usepackage{cite}
\usepackage{hyperref}
\usepackage{multirow}

\newtheorem{theorem}{Theorem}
\newtheorem{problem}{Problem}
\newtheorem{example}{Example}

\newtheorem{proposition}{Proposition}

\newtheorem{lemma}{Lemma}
\newtheorem{definition}{Definition}
\newtheorem{remark}{Remark}

\DeclareMathOperator*{\argmax}{argmax}

\newcommand{\poshs}{\mathcal{H}}

\newcommand{\tree}{\mathcal{T}}

\newcommand{\poalg}{{SaBPI}\xspace}
\newcommand{\nhsalg}{{SaBRS}\xspace}
\newcommand{\Q}{\mathcal{Q}}
\newcommand{\eventually}{\lozenge}
\newcommand{\globally}{\square}

\newcommand{\tempop}[1]{\mathcal{#1}}
\newcommand{\until}{\,\mathcal{U}}
\newcommand{\policytree}{\pi}
\newcommand{\UCB}{\textsc{ucb}}
\newcommand{\ltl}{\textsc{ltl}\xspace}
\newcommand{\ltlf}{\textsc{ltl}_f\xspace}

\title{\LARGE \bf Sampling-based Task and Kinodynamic Motion Planning under Semantic Uncertainty
}

\author{Qi Heng Ho$^{1}$, Zachary N. Sunberg$^{2}$, and Morteza Lahijanian$^{2}$%
\thanks{$^{1}$Q. H. Ho is with the Kevin T. Crofton Department of Aerospace and Ocean Engineering, Virginia Tech, Blacksburg, VA, USA 
    {\tt\small qihengho@vt.edu}}%
\thanks{$^{2}$Z. N. Sunberg and M. Lahijanian are with the Smead Department of Aerospace Engineering Sciences, University of Colorado Boulder, CO, USA 
    {\tt\small \{zachary.sunberg, morteza.lahijanian\}@colorado.edu}}%
}

\begin{document}
    
    \maketitle
    
    \begin{abstract}
        This paper tackles the problem of integrated task and (kinodynamic) motion planning in uncertain environments.  We consider a robot with nonlinear dynamics tasked with a Linear Temporal Logic over finite traces ($\ltlf$) specification operating in a partially observable environment. 
        Specifically, the uncertainty is in the semantic labels of the environment.
        We show how the problem can be modeled as a Partially Observable Stochastic Hybrid System that captures the robot dynamics, $\ltlf$ task, and uncertainty in the environment state variables. We propose an anytime algorithm that takes advantage of the structure of the hybrid system, and combines the effectiveness of decision-making techniques and sampling-based motion planning. We prove the soundness and asymptotic optimality of the algorithm. Results  show the efficacy of our algorithm in uncertain environments, and that it consistently outperforms baseline methods.
    \end{abstract}

    \section{Introduction}
        \label{sec:intro}
        \input{sections/Introduction}


    \section{Problem Formulation}
        \label{sec:problem}
        \input{sections/Problem}

    \section{Approach}
        \label{sec:approach}
        \input{sections/Reformulation}

    \input{sections/Methodology}

    \input{sections/Analysis}

    \section{Evaluation}
        \label{sec:experiments}
        \input{sections/Evaluation}

    \section{Conclusion}
        \label{sec:conclusion}
        \input{sections/Conclusion}
    
    \bibliographystyle{IEEEtran}
    \bibliography{references}
    
\end{document}

%% file: sections/Introduction.tex
Planning for robots with complex nonlinear dynamics to accomplish high-level, long-horizon tasks is a fundamental challenge in robotics. Such tasks are often specified using formal languages like Linear Temporal Logic over finite traces ($\ltlf$)~\cite{ltlf}, which provides a precise and expressive framework for specifying sequences of events and properties over time. While many approaches exist, their applicability becomes far more challenging in uncertain, partially observable environments. In these settings, the robot is uncertain about its progression toward task completion, and the problem becomes a tightly coupled interplay between continuous dynamics, discrete task semantics, and stochastic belief evolution. Reasoning under such uncertainty while guaranteeing bounds on task satisfaction probability makes the problem particularly difficult. This paper develops an efficient framework for $\ltlf$ task and motion planning in uncertain, partially observable environments with formal guarantees on task success probability.

\begin{example}
    \label{ex: drone}
Consider a drone flying in a forested area, depicted in Fig.~\ref{fig: drone shs}, tasked with detecting possible fires at three sites. If fire is detected, the robot should proceed to exit A to report it; otherwise it should proceed to exit B. The presence of fire at each site is unknown, and the drone's sensor provides noisy observations. Observations taken from directly above the sites are less accurate than those taken at closer range, creating a tradeoff between flight path and information quality. Our goal is to enable optimal task and motion planning for such scenarios.
\end{example}

\begin{figure}[ht!]
    \centering
    \includegraphics[width=1.0\linewidth]{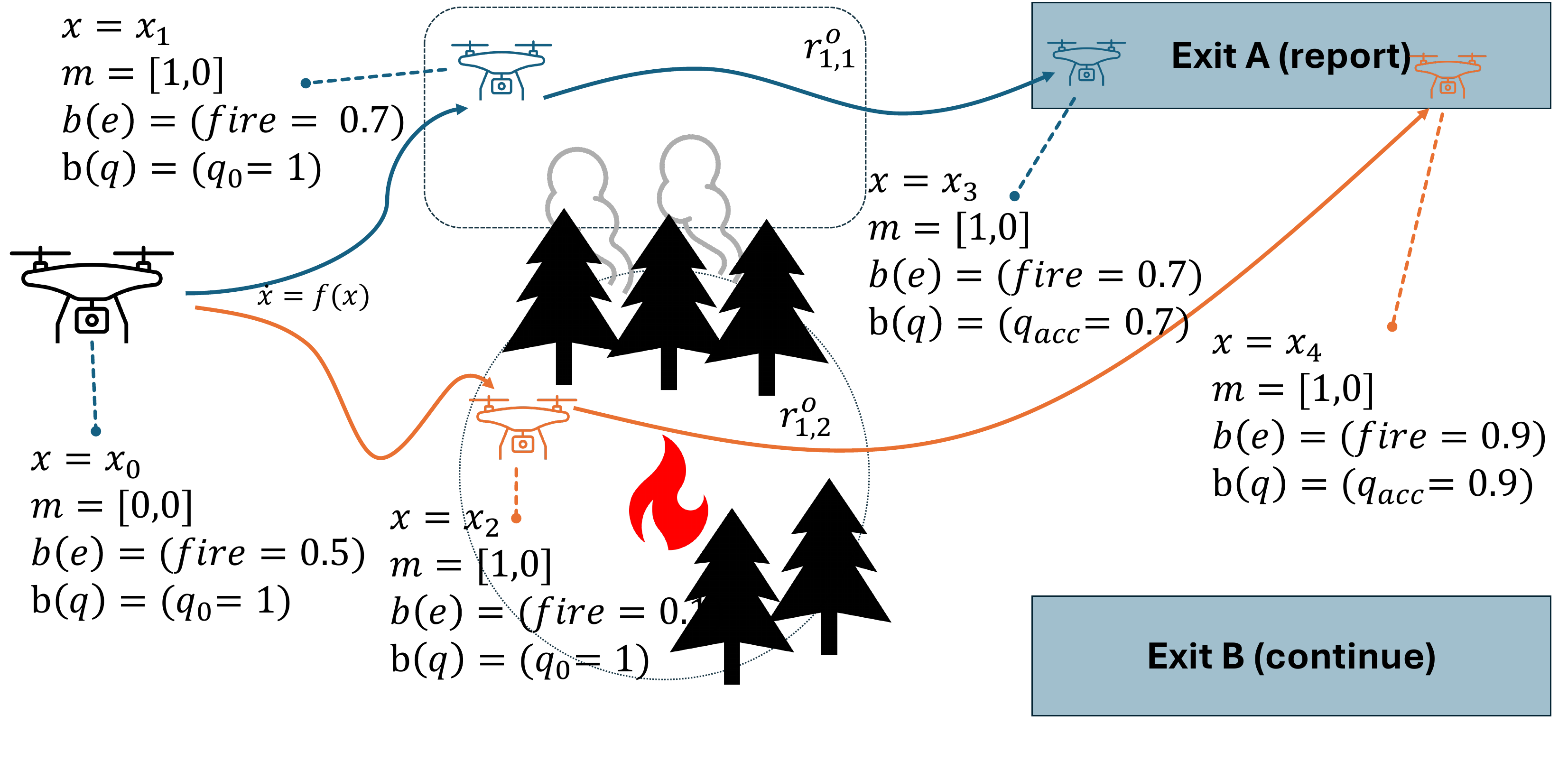}
    \caption{This example shows two possible motion trajectories in Example~\ref{ex: drone}, where the observation quality is better near the fire than above the trees. In both (blue and orange) trajectories, the drone evolves in continuous time until it reaches $r_{1,2}^o$ or $r_{1,1}^o$, which triggers a guard $G_{T}$. Let both observations be $o = fire$ Then, it updates its memory state $m$, and updates its belief $b(e)$ of the environment state. The drone state evolves again until it reaches exit A, in which $G_{R}$ is triggered. This causes a belief update to the task DFA belief $b(q)$. In this example, the orange trajectory has higher probability of success than blue.}
    \label{fig: drone shs}
\end{figure}

Sampling-based motion planning (SBMP) algorithms like RRT~\cite{kinoRRT} are effective for nonlinear dynamics and have been extended to $\ltlf$ specifications~\cite{Luo2021,Lahijanian:TRO:2016}. However, they typically assume full observability or unimodal belief distributions. Partially Observable Markov Decision Processes (POMDPs) provide the standard framework for reasoning under sensor noise and environmental uncertainty, but solving them for indefinite-horizon reachability is generally undecidable~\cite{MADANI2003Undecidability}, and most solvers require discrete state-action spaces or finite-state abstractions that decouple high-level reasoning from low-level control. Neither paradigm alone can handle Stochastic Hybrid Systems (SHS), where continuous time-space evolution and discrete semantic uncertainty are tightly coupled. The key insight of our approach is that bandit-guided policy selection can steer a sampling-based planner through the hybrid belief space efficiently, without requiring discretization or predefined low-level controllers, yielding an anytime algorithm with sound probability guarantees.

In this work, we study the problem of synthesizing motion policies for continuous robotic systems with partial observations in environments with uncertain semantic labels, subject to $\ltlf$  specifications. We model this as a Partially Observable Stochastic Hybrid System (PO-SHS) and propose an algorithm that combines the scalability of sampling-based motion planning with decision-theoretic principles from POMDP solvers to explore the continuous hybrid belief and policy space. Our algorithm operates in an anytime fashion, returning an iteratively improving policy that can be terminated at a user-specified performance threshold or time limit, and we prove it is sound and asymptotically optimal.

In summary, our contributions are fourfold:
(i) A PO-SHS modeling framework for $\ltlf$ task and motion planning under partially observable environment states, enabling policy synthesis with formal guarantees on task success probability.
(ii) A novel algorithm, Sampling-based Bandit-guided Policy Iteration (SaBPI), that uses a bandit-based approach to select promising policies and guide a sampling-based motion planner in the hybrid belief space.
(iii) A theoretical analysis establishing that SaBPI produces an anytime policy that is a sound lower bound on task success probability and is asymptotically optimal. To the best of our knowledge, this is the first such algorithm for hybrid POMDPs with deterministic continuous flow fields. (iv) Empirical evaluation on several case studies and benchmarks demonstrating that our approach outperforms several baselines.

\subsection{Related Work}

The problem of integrated task and motion planning (TAMP) under uncertainty spans formal methods, sampling-based planning, and decision-making under partial observability.

\subsubsection{Integrated Task and Motion Planning} Traditional TAMP frameworks decouple symbolic task planning, typically using PDDL~\cite{Kaelbling2010pddlplanning}, from geometric motion planning~\cite{Kingston2018samplingbased}, and generally assume deterministic transitions or full observability. SBMP algorithms have been extended to $\ltl$ specifications over continuous spaces~\cite{Luo2021, Lahijanian:TRO:2016}, and several works address uncertainty in continuous states or measurements~\cite{Luna:WAFR:2015, Oh2021, Ho2023simba} or uncertain maps~\cite{fu2016optimal}. However, none of these methods handle partial observability over discrete semantic states. Our setting addresses the key challenge where the task progress itself is a hidden variable governed by stochastic observations.

\subsubsection{Reactive Planning and MPC}
Receding-horizon and model predictive control methods have been applied to temporal logic constraints~\cite{Wongpiromsarn2012recedingltl} and co-safe LTL planning with active perception~\cite{kantaros2022perception}. While effective for local control and short-horizon replanning, these methods often rely on open-loop or locally-optimal trajectories. Consequently, they lack the global reasoning required to provide sound guarantees on task satisfaction probability for unbounded-horizon $\ltlf$ objectives over hybrid continuous dynamics, where information-gathering actions must be balanced with task progression.

\subsubsection{Decision Making under Partial Observability} 
POMDPs are the standard framework for planning under sensor noise and environmental uncertainty~\cite{shani2013survey}. Point-based solvers~\cite{Pineau2003pbvi} provide lower-bound guarantees but require discrete state and action spaces. Reachability and $\ltl$ synthesis on POMDPs has been studied for discrete spaces~\cite{Ho2024reachability, andriushchenko2023symbiotic, Wongpiromsarn2012pomdpltlcontrol} and via finite-state abstraction of dynamics~\cite{haesaert2018, haesaert2019}, but these abstractions decouple the discrete and continuous layers in ways that struggle with nonlinear hybrid constraints. 

Sampling-based methods~\cite{Sunberg2018pomcpow} handle continuous controls but lack formal reachability guarantees. Common belief-space planners~\cite{bry2011rrbt, Ho2022gbt} assume unimodal Gaussian distributions, which are insufficient for the multimodal beliefs inherent in semantic uncertainty. Recent work has addressed non-Gaussian beliefs for reach-avoid objectives~\cite{vahs2026safetycriticalcontrolpartialobservability} or $\ltlf$ in belief space~\cite{Ho2023simba}, but these typically focus on continuous state uncertainty. Semantic planners~\cite{sundarsingh2026safeplanningunknownenvironments} provide perceptual guarantees but decouple uncertainty estimation from planning. To our knowledge, no existing method provides end-to-end task satisfaction guarantees for $\ltlf$ tasks involving hybrid dynamics with discrete semantic hidden states.

\subsubsection{Connection to SaBRS} Our work adapts the Sampling-based Reactive Synthesis (SaBRS) algorithm~\cite{Ho2024nhs}, which handles nondeterministic hybrid systems in adversarial (worst-case) settings. SaBPI extends this to the probabilistic, partially observable setting, replacing adversarial reasoning with stochastic policy optimization.

%% file: sections/Problem.tex
In this work, we consider nonlinear robotic systems with complex temporal tasks under uncertainty. Specifically, we focus on uncertainty arising  partial knowledge about the environment and partial observability.

\subsection{Robot Dynamics and Task}
The robot's state $x \in X \subset \mathbb{R}^n$ evolves according to
\begin{equation}
    \label{eq:system}
    \dot{x} = f(x, u)
\end{equation}
where $u \in U \subset \mathbb{R}^m$ is the control input,
and $f: X \times U \to \mathbb{R}^n$ is the vector field. We assume  
$f$ is Lipschitz continuous in both arguments, i.e., there exists constants $L_x, L_u > 0$ such that  $\forall x_1, x_2 \in X$ and $\forall u_1, u_2 \in U$,
\begin{align*}
    \|f(x_1, u_1) - f(x_2, u_2) \| \leq L_x \|x_1 - x_2 \| + L_u \|u_1 - u_2\|.
\end{align*}

Let the initial state at time $t_0$ be $x_{t_0} = x_0$. Given a controller $u_t$, $ t \in [0,T]$, the robot's state $x_t$ at time $t$ is the solution to \eqref{eq:system}.

The state space $X$ consists of obstacles $X_{obs}$ (e.g., out-of-bound velocity) and a set of regions of interest $X_r = \{r_1, \ldots, r_{|X_r|}\}$, where $r_i \subset X$. Each region $r_i$ is associated with a set of properties, represented by atomic propositions.
Let $AP = \{p_1, \ldots, p_{|AP|}\}$ denote the set of all atomic propositions.
The labeling function $L: X_r \to 2^{AP}$ maps each region $r_i$ to the subset of properties it satisfies, i.e., $L(r_i) \subseteq AP$. 

The robot is tasked with satisfying temporally extended specifications over $AP$.
We use Linear Temporal Logic over finite traces ($\ltlf$) to express such tasks, as it provides a natural formal language for missions that can be completed in finite time.

\begin{definition}[$\ltlf$ Syntax \cite{ltlf}]
 Let $AP$ be a set of atomic propositions. Then, an $\ltlf$ formula over $AP$ is recursively defined by
    \begin{align*}
        \phi :=  p \mid \neg \phi \mid \phi \vee \phi \mid  \tempop{X} \phi \mid \phi  \until \phi
    \end{align*}
where 
$p \in AP$ is an atomic proposition, $\neg$ (``negation'') and $\vee$ (``or'') are Boolean operators, and  $\tempop{X}$ (``next'') and $\until$ (``until'') are the temporal operators.
\end{definition}

From this definition, one can derive other standard temporal operators, e.g., the $\eventually$ (``eventually") and $\globally$ (``globally") operators are given by $\eventually \phi \equiv \top \until \phi$ and $\globally \phi \equiv \neg \eventually \neg \phi$, where $\top \equiv p \lor \neg p$ is \emph{true}.

$\ltlf$ formulas are interpreted over finite words in $(2^{AP})^*$.  Due to space constraints, we refer the reader to work \cite{ltlf} for the semantics of $\ltlf$. The satisfaction of $\phi$ by a finite word $\sigma \in (2^{AP})^*$ is written as $\sigma \models \phi$.

Consider a trajectory $\mathbf{x}_t$ for $t \in [0,T]$, which visits the regions in $X_{r}$ at time intervals $[t_i, t_{i+1})$ for $i = 0,\ldots,N$ such that $t_{N+1} = T$. The trace (word) of this trajectory is defined as $L(\mathbf{x}) = L(\mathbf{x}(t_0))L(\mathbf{x}(t_1))\cdots L(\mathbf{x}(t_N))$. We say that a robot trajectory satisfies formula $\phi$, written $\mathbf{x} \models \phi$, iff $ L(\mathbf{x}) \models \phi$.

\begin{example}
    Consider the drone in Example~\ref{ex: drone}. The task is 
    \begin{align*}
        \phi = \globally (fire \rightarrow \eventually (A)) \wedge \globally (\neg fire \rightarrow \eventually (B)) \wedge \globally(\neg obs).
    \end{align*}
\end{example}

\paragraph*{$\ltlf$ to Deterministic Finite Automaton (DFA)}

An $\ltlf$ formula $\phi$ can be translated into a DFA  that represents precisely the traces that satisfies $\phi$ \cite{ltlf}. A DFA constructed from $\phi$ is defined as a tuple $\automaton = (Q, q_0, 2^{AP}, \delta, F)$, where $Q$ is a finite set of states, $q_0 \in Q$ is the initial state, $2^{AP}$ is a set of input symbols, $\delta : Q \times 2^{AP} \rightarrow Q$
is the transition function, $F \subseteq Q$ is the set of accepting states.

\subsection{Label and Observation Uncertainty}

As in real-world settings, we assume the robot has incomplete information and imperfect sensing.
In particular, the true atomic propositions associated with each region $r_i$ (i.e., the label $L(r_i)$) are unknown to the robot.
Instead, the robot has an initial belief of the labels and can obtain noisy observations of the label of $r_i$ through a set of sensing regions associated with it, denoted $R^o_i = \{r^{o}_{i,1}, \ldots, r^{o}_{i,|R^o_i|}\}$,  $r^{o}_{i,j} \subseteq X$, e.g., 
the robot may move to $r^{o}_{i,1}$ to capture an image of $r_i$ and then apply a classification algorithm to determine whether fire is present. 
Let $X^o = \bigcup _{i=1}^{|X_r|} R^o_i$ denote the set of all observation regions.

In continuous-time and continuous-space systems, a stochastic observation kernel leads to an infinite-dimensional belief space that is computationally prohibitive for formal policy synthesis. Instead, we treat observations as discrete, state-triggered events, allowing us to map the sensing process to a tractable discrete update. A single observation in our model acts as an abstraction of a continuous sensing process where multiple sensor measurements are abstracted into a single localized inference. This abstraction aligns with the output of modern perception modules that aggregate temporal data into semantic labels or classifications \cite{Kanso2025semanticslam}. Therefore, we assume a \emph{persistent} (or ``sticky") observation model: repeated observations within the same sensing region $r^o_{i,j}$ provide no additional information beyond the initial measurement\footnote{In many robotic applications, perception errors are driven by static environmental conditions (e.g., occlusions or fixed lighting). Because these errors are systematic and spatially correlated, repeated measurements (e.g., range or infrared sensors) from the same viewpoint yield redundant information.}. To this end, we define a memory vector $m \in \{0,1\}^{|X^o|}$, where the $k$-th element $m_k = 1$ if the $k$-th region in $X^o$ has been visited; otherwise, $m_k = 0$.

We model this observation process using a  Persistent Observation Hidden Markov Model (HMM). 

\begin{definition}[PO-HMM]
\label{def: hmm}
A Persistent Observation Hidden Markov Model (PO-HMM) is a tuple 
$\mathcal{M} = (X, E, , X^o, M, O, Z, T, b(e_0,m_0))$, where 
\begin{itemize}
    \item $X$ is the robot's state space,
    \item $E = 2^{AP} \times X_r$ is the set of all (hidden) environment states,
    \item $X^o$ is the set of all observation regions,
    \item $M = \{0,1\}^{|X^o|}$ is the set of all memory vectors associated with visiting the elements of $X^o$,
    \item $O = 2^{AP}$ is the set of observations (labels),
    \item $Z : X \times X^o \times E \times M \times O \rightarrow [0,1]$ is the observation function, such that 
    %
    given the current robot state $x \in X$ at
    the observation region $r^o_i \in X^o$ associated with $r_i$,
    the hidden state $e=(r_i, L(r_i))$ determined by the true label $L(r_i)$, and
    the memory vector $m \in M$, 
    the robot observes label $o \in O$ for region $r_i$ with probability $Z(o \mid x, r^o_i, e, m)$.
    If the $i$-th entry of $m$ is 1, i.e., $r^o_i$ is previously visited, the observation is uninformative (null observation)\footnote{This can also be modeled with $
    Z(o \mid x, r^o_i, e, m) = \frac{1}{|O|}$ if $m_i = 1$.},
    \item $T : M \times X \times  X^o \rightarrow M$ is the memory update function such that for 
    $m'=T(m,x,r^o_i)$,
    $m_i' = 1$ if $x \in r^o_i$, otherwise, $m'_i = m_i$,
    where $m_i$ denotes the $i$-th element of $m$.
    \item $b(e_0, m_0) \in \Delta(E) \times M$, is the robot's initial belief of the labels of the regions in $X_r$, where $\Delta(E)$ denotes the set of all probability distributions over the elements of $E$, and $m_0$ is the initial memory vector.
\end{itemize}
\end{definition}

Instead of knowing the exact environment states $e$, the agent maintains a \emph{belief state}, which is a probability distribution over all possible environment states and memory, conditioned on the history of observations. We denote the belief at time $t$ as $b^{e, m}_t \in \Delta(E) \times M$.

\subsection{Planning Problem}

We are interested in decision making (what regions to visit) and motion planning (what controls to apply) for this robot to achieve its task $\phi$. We capture both of these aspects using a motion policy.

\begin{definition}[Motion Policy]
    Let 
    $h_t = (b_0,\langle r^o_1, o_1\rangle, \langle r^o_2, o_2\rangle, \ldots, \langle r^o_n, o_n\rangle) \in H$ 
    be the historical sequence of visited observation regions and received observations up to time $t$. An \emph{observation feedback-motion policy} is the mapping $\pi: H \times X \rightarrow U$.
\end{definition}

Given a motion policy $\pi$, a probability measure over the trajectories of the robot is induced by PO-HMM \cite{Skovbekk:TAC:2025}.  Hence, the satisfaction of task $\phi$ under $\pi$ becomes probabilistic, denoted by $\mathbb{P}(\pi \models \phi)$.
Our objective is to synthesize a $\pi$ that maximizes the probability of satisfying $\phi$.

\begin{problem}[Motion Policy Planning]
    \label{prob: policy}
    Given the system dynamics in \eqref{eq:system} with 
    the PO-HMM observation model in Definition~\ref{def: hmm}, and an $\ltlf$ formula $\phi$, compute a motion policy $\pi^*$ that maximizes the probability of satisfying $\phi$, i.e.,
    \begin{align}
        \pi^* = \argmax_{\pi} \;\; \mathbb{P}(\pi \models \phi).
    \end{align}
\end{problem}

Problem~\ref{prob: policy} is challenging for several reasons.
First, uncertainty in the observation of atomic propositions requires computing an observation-history–based policy, i.e., solving a POMDP.
Second, generating controls for a nonlinear robotic system to visit the required regions and complete the task is a kinodynamic motion planning problem.
Even simplified versions of these problems, i.e., geometric motion planning or policy synthesis for a finite-horizon discrete POMDP, are PSPACE-complete.

Standard methods for either motion planning or POMDPs are ill-suited for this combined setting, due to continuous-time dynamics and partial observability.
Our approach mitigates this intractability by formulating the problem as a hybrid of motion planning and decision making, and leveraging elements of both to design a provably \emph{sound}, \emph{anytime} algorithm that efficiently computes high-quality policies.

%% file: sections/Reformulation.tex
In our approach, we first convert the problem into a partially observable stochastic hybrid system (PO-SHS), and then perform policy synthesis on it.

\subsection{Partially Observation Stochastic Hybrid System}

\paragraph*{PO-SHS Construction}

We reformulate the planning problem in Section~\ref{sec:problem} as policy synthesis in a Partially Observable Stochastic Hybrid System with discrete uncertainty. This is achieved by taking the cartesian product of the robot state and dynamics in \eqref{eq:system}, PO-HMM $\mathcal{M}$, and DFA $\mathcal{A}$ of the $\ltlf$ formula $\phi$. 

\begin{definition}[PO-SHS with Persistent Observations]
\label{def:POSHS_full}
Given System~\eqref{eq:system} with PO-HMM $\mathcal{M} = (X, E, , X^o, M, O, Z, T, b_0)$ observation model and DFA $\automaton = (Q, q_0, 2^{AP}, \delta, F)$,
a \emph{partially observable stochastic hybrid system with persistent observations} (PO-SHS) is defined as $\poshs = \mathcal{M} \times \automaton = (S, b^s_0, U, f, \delta, X_{AP}, Z, T, S_\text{acc}),$ where
\begin{itemize}
    \item $S = Q \times E \times M \times X$ is the hybrid state space with discrete modes $Q \times E \times M$ and continuous state space $X$;
    \item $b^s_0 \in \Delta( \{q_0\} \times E) \times M \times X$ is the initial belief;
    \item $U$ is the control space;
    \item $f = X \times U \rightarrow \mathbb{R}^n$ is the continuous dynamics \eqref{eq:system};
    \item $G_{R}: X \rightarrow \{\top, \bot\} $ is the set of guard regions associated with $X_r$ such that $G_{R}(x) = \top \iff x \in X_r$;
    \item $\delta : Q \times 2^{AP} \rightarrow Q$ is the deterministic DFA transition function.
    \item $G_{T}: X \rightarrow \{\top, \bot\}$ is the set of guard regions associated with $X^o$ such that $G_{T}(x) = \top \iff x \in X^o$;
    \item $T : M \times X \times  X^o \rightarrow M$ is the memory update function as defined in Def.~\ref{def: hmm};
    \item $Z : X \times X^o \times E \times M \times O \rightarrow [0,1]$ is the observation function as defined in Def.~\ref{def: hmm};
    \item $S_\text{acc} =  F \times E \times M \times X \subseteq S$ is the set of accepting hybrid states, i.e., those whose DFA component is accepting.
\end{itemize}
\end{definition}

This definition is a special case of the typical partially observable stochastic hybrid system, where the transition function between modes is deterministic, the reset map is identity, and there is no stochasticity in the flow function.

Note that the uncertainty due to partial observability is only over $Q \times E$ components of the hybrid state, i.e., hybrid state $s_t = (q_t,e_t,m_t,x_t) \in S$ at time $t$ is a random variable because $q_t \in Q$ and $e_t \in E$ are partially observable.  Hence, the belief of $s_t$ is denoted by $b^s_t \in B = \Delta(Q \times E) \times M \times X$.  We refer to $B$ as the belief space of $\poshs$. 

A policy in $\poshs$ can be represented via belief states, which are sufficient statistics for the history.

\begin{definition}[Belief-based Control Policy]
    A belief-based control policy for PO-SHS $\poshs$ is the mapping $\pi_{\poshs} : B \rightarrow U$.
\end{definition}

Under $\pi_{\poshs}$, the evolution of the system in $\poshs$ is as follows. At the continuous time $t$, the agent has a belief $b^s_t \in B = \Delta(Q \times E) \times M \times X$, which consists of deterministic robot state $x_t \in X$ and memory $m_t \in M$, and a distribution over DFA state $q_t \in Q$ and environment state $e_t \in E$. 
The continuous state $x_t$ of $s_t$ evolves according to \eqref{eq:system} under $\pi_{\poshs}$ until a guard $G_T$ or $G_R$ is triggered. Note that until a guard is triggered, $q_t$, $e_t$, and $m_t$ remains constant, and consequently the distribution over $(q_t,e_t)$ does not change.
Let $\tau$ denote the time instant at which the system first hits a guard, i.e., $G_{R}(x_{\tau}) = \top$ or $G_T(x_{\tau}) = \top$. Then, the discrete components of $s_\tau$ makes an instantaneous transition. If $G_R$ is triggered at $r^o$, the new distribution of $Q$ is updated according to $\delta$ as detailed below. If $G_{T}$ is triggered, then the memory state is updated, and the system receives an observation $o \in O$ according to $Z$. The agent then updates its belief through a discrete jump at this time instance $\tau$. 
Let $b^s_{\tau^-}$ and $b^s_{\tau^+}$ denote the beliefs before and after the instantaneous discrete jump at time $\tau$, respectively. The belief update is given by
\begin{multline}
        \label{eq:beliefupdate}
 b^s_{\tau^+}(s^+) \propto Z(o | x_{\tau^-}, r^o, e_{\tau^-}, m_{\tau^-})\cdot\\
     \sum_{s^- \in Q\times E\times \{m_{\tau^-}\} \times \{x_{\tau^-}\}}\!\!\!\!\!\!\!\!\!\!\!\!\mathbf{1}(q^+,q') \mathbf{1}(m^+, m') b^s_{\tau^-}(s^-).
\end{multline}
\noindent where $q' = \delta(q^-, L(x_{\tau^-}))$ and $m' = T(m^-, x_{\tau^-}, r^o)$, and $\mathbf{1}(y,y')$ is the indicator function such that $\mathbf{1}(y,y') = 1 \iff y = y'$, otherwise $0$. 

\begin{example}
    Fig.~\ref{fig: drone shs} shows example belief trajectories for the PO-SHS for Example~\ref{ex: drone}.
\end{example}

The full construction of $\poshs$ leads to a combinatorial explosion in the size of the hybrid state space, since the number of discrete modes in the state space is $|Q| \cdot 2^{|AP|} \cdot |X_r| \cdot |X^o|$. Instead of explicitly constructing $\poshs$, we implicitly construct it by growing a tree via its semantics starting at $b^s_0$. In the following, we discuss our tree-based search algorithm to find piecewise constant $\pi_\poshs$ for $\poshs$.
\begin{remark}
    While Problem \ref{prob: policy} can be equivalently formulated as a POMDP, it does not admit a finite-state representation. Standard POMDP solvers typically require a discount factor or a finite horizon to ensure convergence, or they rely on online search heuristics that lack formal soundness guarantees during deployment. By contrast, the PO-SHS structure defined here allows for a tailored offline algorithm that explicitly leverages the system's hybrid dynamics to provide rigorous guarantees on the task success probability.
\end{remark}

%% file: sections/Methodology.tex
\section{Anytime Sampling-based Bandit-guided Policy Iteration}
\label{sec:methodology}

In this section, we present our methodology to policy synthesis in $\poshs$, which solves Problem~\ref{prob: policy}. We propose an algorithm, called \emph{Sampling-based Bandit-guided Policy Synthesis} (\poalg). \poalg is a novel adaptation of the core principles of the Sampling-based Bandit-guided Reactive Synthesis (SaBRS) algorithm \cite{Ho2024nhs}. We shift its focus from a nondeterministic two-player game to a probabilistic, partially observable setting. The key conceptual change is our re-interpretation of the search tree and the value function. Instead of minimizing a worst-case cost against an adversarial player, \poalg maximizes the probability of satisfying $\phi$ by propagating and updating beliefs in $\poshs$.

\poalg iteratively grows a search tree in the hybrid belief space of $\poshs$ by taking advantage of its specific structure. The tree $\tree = (\mathcal{N},\mathcal{E})$ is rooted at the initial belief state $b^s_0$ and extended based on the semantics of $\poshs$. Each node $n \in \mathcal{N}$ of the tree is a tuple $n = \langle b^s, N, \mathbf{u} \rangle$, where $b^s$ is a hybrid belief, $N$ is the number of times that node $n$ has been visited, and $\mathbf{u} = \{(u_i, t^{prop}_i)_i|i = 1,...,k\}$ is the set of piecewise constant inputs previously selected at $n$.

We model this search tree as an \texttt{AND/OR} tree.  \texttt{OR} nodes correspond to belief states, where the robot chooses a control input $(u,t) \in n.\mathbf{u}$. \texttt{AND} nodes correspond to the stochastic outcomes of applying that control. Each outgoing edge $\mathrm{e} = ((n,(u,t)),n') \in \mathcal{E}$ is given a probability $w(\mathrm{e}) = \mathbb{P}(n' \mid n,(u,t))$ that follows the semantics of $\poshs$. A policy tree on the \texttt{AND/OR} structure corresponds to a piecewise-constant control policy in the hybrid belief space. The objective is to construct a policy tree that maximizes the probability of reaching the accepting set $S_{acc}$, thereby satisfying $\phi$. 

To construct this policy tree efficiently, we propose an algorithm that alternates between two main steps: \emph{policy subtree selection} and \emph{policy expansion}. \poalg uses a bandit-based action selection methodology to select a policy in the tree. Then, the policy expansion step uses motion planning sampling techniques to grow the selected subtree. The pseudocode for \poalg is shown in Alg.~\ref{alg:framework}. This alternation of selecting promising policy subtrees and expanding on them combines the exploration-exploitation properties of bandit-based techniques at the policy level with the effectiveness of sampling-based motion planners in exploring a metric search space, resulting in an efficient, sound, and asymptotically optimal algorithm (see Sec.~\ref{sec:analysis}). The algorithm is run in an anytime fashion and ends when a termination criterion is met, i.e., a planning time threshold is hit, or a policy with success probability $1$ (or one that meets a success probability threshold) is found.

\begin{remark}
    This approach is similar to
    common tree-search algorithms in planning under uncertainty (i.e., MDPs and POMDPs). The main difference lies in how the policy space is searched. Instead of adding a single node from the search tree every iteration, we use sampling-based motion planning techniques to expand a selected policy tree many times each iteration by taking advantage of the semantic structure of $\poshs$.
\end{remark}

\begin{algorithm}[t]
    \caption{\poalg Algorithm}
    \label{alg:framework}
    \SetKwInOut{Input}{Input}\SetKwInOut{Output}{Output}
    \Input{POSHS $\poshs$, Initial belief $b_0$, Expansion ratio $k$, Exploration constant $c$}
    \Output{Anytime policy $\pi^{*}$ from $b_0$}
    $n_0 \gets \langle b_0, 0, \emptyset \rangle$\\
    $\tree = (\mathcal{N} \gets \{n_0 \}, \mathcal{E} \gets \emptyset)$\\
    \While{termination criteria not met}
    {
        $\policytree_\poshs \gets \texttt{UCB-ST}(n_0, c)$\\
    \For{$j = 1 \rightarrow k$}
    {
        $\policytree_\poshs \gets \texttt{Explore}(\policytree_\poshs)$ 
    }
    $\tree \gets \tree \cup \policytree_\poshs$\\
    $\hat{\policytree}_{\poshs}^* \gets \texttt{UCB-ST}(n_0, 0)$
    }
    \Return $\hat{\pi}^*$
\end{algorithm}

\subsection{Policy Evaluation}

To conduct tree search on $\mathcal{T}$, we have to evaluate a policy $\pi_{\poshs}$. In $\poshs$, the satisfaction of $\phi$ reduces to reachability to the set $S_{acc}$, which are sink states. Thus, we define a value function that represents the probability of reaching $S_{acc}$. We first define a value function for states:
\begin{align}
    \label{eq:value function}
    V(s) = \begin{cases}
        1 & \text{ if } s \in S_{acc}\\
        0 & \text{ if } s \notin S_{acc}
    \end{cases}
\end{align}

Using the value function, we define value functions for belief nodes given a policy as:
\begin{align}
    \label{eq:value function nodes}
    V^\pi(n) = \mathbb{E}[V^\pi(n')]
\end{align}

\noindent where $n'$ is a successor state of $n$. Intuitively, the value of $n$ under policy $\pi$ is the probability of successor nodes that are in $S_{acc}$.

\subsection{Policy Subtree Selection}

The search tree is made up of a set of policies, in the form of \texttt{AND/OR} subtrees. During search, the set of policies in the search tree can become extremely large. Hence, in each iteration of planning, the algorithm selects a policy subtree $\policytree_\poshs$ of the full search tree for expansion. The key technique is to use bandit-based exploitation/exploration to bias our tree expansion towards the policies that are promising, i.e., have high probability of satisfying $\phi$. 

From \eqref{eq:value function nodes}, the \emph{$\Q$-value} is the value for choosing input $(u,t)$ at node $n$ and then following policy $\pi$ at subsequent nodes,
\begin{align}
    \Q^{\pi}(n,(u,t)) = \mathbb{E}[V^{\pi}(n')]
\end{align}

The $\Q$-value allows us to evaluate the relative value of each input $(u,t^{prop})$ at node $n$. The maximization of $\Q$-value at each node provides us with the current best policies at each node (exploitation):
\begin{align}
    \label{eq: maxQ}
    V^{{\pi}}(n) = \max_{(u,t)} Q^{{\pi}}(n, (u,t)).
\end{align}
However, we also want to choose other policies that may eventually be optimal (exploration). We use a variant of the algorithm \emph{UCB for Strategy Tree selection} proposed by \cite{Ho2024nhs}. This algorithm chooses a policy subtree by treating the control selection problem at each node as a separate multi-armed bandit problem to tackle this exploration-exploitation tradeoff. For completeness, the algorithm is shown in Alg.~\ref{alg:UCB-PT}.

Alg.~\ref{alg:UCB-PT} first initializes the chosen policy tree $\policytree$ with the root node $n_0$ (Line $4$ Alg.~\ref{alg:framework}). From $n_0$, it selects control inputs $(u,t)_{sel}$ in the \texttt{AND/OR} tree according to the UCB1 criterion (Line 3 in Alg.~\ref{alg:UCB-PT}) from \eqref{eq:UCB1}:
\begin{align}
    \label{eq:UCB1}
    \UCB(n, (u,t)) = \Q^{{\pi}}(n,(u,t)) + c\sqrt{ 2 \ln (n.N) / n.N_{(u,t)} },
\end{align}
\noindent where $\max_\pi \Q^{{\pi}}(n,(u,t))$ is the $\Q$-value of the best policy found from node $n$ and taking control $(u,t)$, and $n.N_{(u,t)}$ is the number of times the control input $(u,t)$ has been selected at node $n$. All children (chance) nodes of $(n,(u,t)_{sel})$ are added to $\policytree$, and control inputs for each children are again selected according to \eqref{eq:UCB1} (Lines 5-6). This process repeats until a policy subtree $\pi$ of $\tree$ is obtained, which is when all the leaf nodes of a subtree are reached. \texttt{UCB-ST} biases selecting more promising parts of the tree, while still enabling exploration of subtrees that may eventually be part of the optimal policy.

\begin{algorithm}[ht]
    \caption{\texttt{UCB-ST}($n, e$)}
    \label{alg:UCB-PT}
     $\pi \gets \{n\}$; $n.N = n.N + 1$\\
    \If{$n$ is not a leaf node}
    {
        $(u,t)_{sel} \gets \argmax_{(u, t) \in n.u} \UCB(n, (u,t))$ using \eqref{eq:UCB1} \\
        $\pi_{\poshs}(n) = (u,t)_{sel}$\\
        \For{$n' \in children(n, (u,t)_{sel})$}
        {
            $\policytree_\poshs \gets \policytree_\poshs \cup \texttt{UCB-ST}(n', e)$
        }
    }
    \Return $\pi_\poshs$
\end{algorithm}
\subsection{Policy Improvement}
\label{sec:tree extension}

A policy subtree $\policytree_\poshs$ is extended in a sampling-based motion planning manner in the hybrid belief space. Any tree-based kinodynamic sampling-based technique (e.g., RRT or EST~\cite{Kingston2018samplingbased}) can be used in this step. The pseudocode for this step is shown in Alg.~\ref{alg:motionplanner}. In each iteration of exploration, a node $n$ in $\policytree$ that has a value less than one is sampled (nodes which have a value equal to one already have optimal policy subtrees). Likewise, nodes that have all its probability mass that violate $\phi$ have $0$ probability of eventually completing the task and thus are not sampled.

Let the sampled node $n$ have a belief state $b^s= (b^q, b^e, m, x)$. Then, a control $u \in U$ and time duration $t^{prop}$ are randomly sampled, and the node's continuous state $x$ is propagated by its dynamics. During propagation at continuous state $x'$, if guard $G_{R}$ is triggered, or guard $G_{T}$ is triggered for the first time ($x \in r_i \in X^o \wedge m_i = 0$), propagation is terminated. For a guard $G_{R}$, the belief $b^s$ is updated according to the DFA transition function $\delta$. For a guard $G_{T}$ associated with sensing regions, the memory is updated according to $T$, and a new belief is created for every $o \in O$ with non-zero probability according to $\mathbb{E}_{e\sim b^e}[Z(o|x, x^o, e, m)]$. Then, the belief of the state (i.e., belief of $q$ and $e$) is updated according to the Bayes rule \eqref{eq:beliefupdate}. If no guard is triggered and the state is propagated for the full duration, one new node is created and $q, e, m$ stays constant. Then, the control and propagated duration $(u,t_{act})$ 
is added to the set of controls $n.\mathbf{u}$, the probability of reaching that node is added to the edge, and the edge and leaf node are added to the tree. Finally, the value of nodes in the policy is updated by backpropagation to the root using \eqref{eq:value function nodes}. 

This expansion step is repeated $k$ times for each policy subtree selection iteration to ensure that sufficient expansion of a policy subtree is performed at each selection iteration. 

\begin{algorithm}[t]
    \SetKwInOut{Input}{Input}
    \SetKwInOut{Output}{Output}
    \caption{\texttt{Explore}($\policytree$)}
    \label{alg:motionplanner}
    $n_{sel} \gets$ \texttt{SampleAndSelect}($\policytree$)\\
    $u_{rand} \gets \texttt{SampleControl}(U$)\\$t_{rand} \gets \texttt{SampleDuration}((0, T_{prop}]$)\\
    $n_{select}(\mathbf{u}) \gets n_{select}(\mathbf{u}) \cup \{(u_{rand}, t_{rand})\}$\\
    $\{(n_{i}, \mathbb{P}(n_{i}| \cdot))_i|i,...,m\}, t_{act} \gets$ \texttt{Prop}($n_{sel}, u_{rand}, t_{rand}$)\\
    \For{each $n_{i}$}
    {   
    $\policytree(n) \gets \policytree(n) \cup n_{i}$\\
    $\policytree(e) \gets \policytree(e) \cup ((n, (u_{rand},t_{act})), n_{i})$\\
    Label $w(e)$ with $\mathbb{P}(n_{new}| \cdot)$\\
    Update node values from $n_{new}$ to $n_0$ by backpropagation to root via \eqref{eq:value function nodes}
    }
    \Return $\policytree$
\end{algorithm}

%% file: sections/Analysis.tex
\section{Analysis}
\label{sec:analysis}
In this section, we show that the solution returned by our anytime algorithm is a \emph{sound lower bound} on the optimal solution. Additionally, our algorithm asymptotically converges to the optimal probability of satisfying $\phi$. To be concrete, we consider the case that kinodynamic RRT \cite{kinoRRT} is used as the strategy expansion technique. First, we state our first main result, which is that our algorithm is \emph{sound}.

\begin{lemma}[Sound Policy]
    At iteration $i$, \poalg produces a sound lower bound policy $\pi_i$, i.e., the reported value $V^{\pi_i} \leq V^{\pi^*}$, and executing $\pi_i$ results in at least that expected return.
\end{lemma}
\begin{proof}
    The partially observable stochastic hybrid system in Definition~\ref{def:POSHS_full} is constructed exactly via formulation in Problem~\ref{sec:problem}. In the initial instantiation $b_0$, we have that $V^{\pi_0}(b_0) = \mathbb{P}_{s\in b_0}[s \in S_{acc}] \leq V^*(b_0)$, and $V^{\pi_0}(b_0)$ is the exact expected return under $\pi_0$. Suppose after $i$ iterations of node additions, the algorithm maintains $V^{\pi_i} \leq V^*(b)$. In iteration $i+1$, a new leaf node $n_{leaf}$ is added. The leaf node has value $V(n_{leaf}) = \mathbb{P}_{s \in b_{leaf}}[s \in S_{acc}]$. Consider the backpropagation at a node $n$ from leaf node $n_{leaf}$ with \eqref{eq: maxQ} to the root node $n_0$. This is an exact Bellman backup, and we have that $V^{\pi_{i+1}}(n) \leq V^*(n) \forall n \in \mathcal{T}$, and $V^{\pi_{i+1}}(n)$ is the exact expected return. By induction, every $V^{\pi_i}$ computed by \poalg is sound.
\end{proof}

Next, we show that our policy selection methodology repeatedly selects every policy.

\begin{proposition}(Lemma 1 in \cite{Ho2024nhs})
    \label{lemma:UCB-ST}
    Given a search tree $\tree$ that grows over iterations and an exploration constant $c > 0$, 
    \texttt{UCB-ST} in Alg.~\ref{alg:UCB-PT} selects every policy subtree that is introduced at some finite iteration infinitely often. 
    That is, as the number of iterations approaches infinity, the number of times each such subtree of $\tree$ is selected also approaches infinity.
\end{proposition}

We assume that there exists an optimal policy with a finite branching of finite length trajectories. Further, similar to \cite{Ho2024nhs}, we assume that an optimal solution with a non-zero radius of clearance exists, as defined below.

\begin{definition}[Clearance]
    For policy $\pi$, let $Traj^\pi$ be the trajectories of $\poshs$ induced under $\pi$ with non-zero probability and $B^\pi = Q^\pi \times \Delta(E)^\pi \times M^\pi \times X^\pi \subset B$ be the set of hybrid beliefs visited by $Traj^\pi$. Further, denote by $\mathbb{B}_{\delta}(x)$ the ball centered at point $x$ with radius $\delta \geq 0$.  
    Clearance $\delta_{clear}^\pi$ of $\pi$ is the supremum of radius $\delta$
    such that for every $b^{s}(s = (q,e,m,x)) \in B^\pi$ and $\forall x' \in \mathbb{B}_{\delta}(x) \subset X$, it holds that $G_{R}(x') = G_{R}(x)$ and $G_{T}(x') = G_{T}(x)$.
\end{definition}

We now formally state the second main result of our analysis, which is that \poalg (Alg.~\ref{alg:framework}) asymptotically approaches the optimal policy.

\begin{theorem}[Asymptotically Optimal]
    As the number of \poalg iterations $i \rightarrow \infty$, $V^{\pi_i} \rightarrow V^*$, i.e., algorithm is asymptotically optimal.
\end{theorem}

\begin{proof}
    The proof follows a similar argument as probabilistic completeness in \cite{Ho2024nhs}. Consider an optimal policy $\policytree^*$ with clearance $\delta_{clear} > 0$. A trajectory $p_i \in Traj^{\policytree^*}$ can be described by the mapping $p_i : [0, t^{p_i}] \rightarrow B$. Let the sequence of beliefs from the trajectory $p_i$ be $b_0,b_1,\dots, b_{acc}$, where two beliefs $b_i, b_j$ have any interval duration $\tau > 0$ such that a guard region is not hit between $b_i$ and $b_j$, and cover $p_i$ with a set of balls of radius $\delta$ centered at $b_0, b_1, \cdots, b_{acc}$, such that for every $b^s (q,e,m,x) \in p_i$ and $\forall x' \in \mathbb{B}_{\delta}(x) \subset X$, it holds that $G_{R}(x') = G_{R}(x)$ and $G_{T}(x') = G_{T}(x)$. We say that a path $p_j$ \emph{follows} another path $p_i$ if each vertex of $p_j$ is within this constrained $\delta$ radius ball of $p_i$. Since the flow functions $F$ and jump functions $J$ are Lipschitz continuous, from \cite[Theorem 2]{rrtpc}, we are guaranteed that RRT asymptotically almost surely finds a control trajectory from $b_0$ to $b_{acc}$ that follows $p_i$ when starting from a tree which contains $b_0$. From Lemma~\ref{lemma:UCB-ST}, \texttt{UCB-ST} will always eventually select any strategy subtree $\policytree_i$ of our search tree $\tree$. Let $t$ be the number of paths in $\policytree$ that uniquely follows a path $p_i \in P$. Assume that at step $j$, the selected subtree $\policytree$ contains $0 < t <|P|$ paths. As expansion iterations increase, $\policytree$ asymptotically almost surely finds a path from $s_0$ to $b_{acc}$ that follows a new path $p_l \in P$ which it did not uniquely follow before. Hence, the new $\policytree^+_{i}$ expanded from $\policytree_{i}$ now contains $t + 1$ paths that unique follows paths in $P$. From Lemma~\ref{lemma:UCB-ST}, $\policytree^+_i$ will eventually be selected again. By induction, $t \rightarrow |P|$ and an optimal policy is found.
\end{proof}


%% file: sections/Evaluation.tex
We evaluate our proposed methodology in this section. 

\textbf{Setup: } We evaluate the performance of \poalg against $3$ algorithms, namely kinodynamic RRT, Monte Carlo Tree Search with progressive widening (MCTS-PW), and the nondeterministic hybrid system planner (\nhsalg) from \cite{Ho2024nhs}, in a series of case studies. For \nhsalg, we aim to evaluate how an algorithm designed for nondeterminism and not stochastic uncertainty would perform. Thus, we output an winning ($\mathbb{P}(\pi \models \phi)$) policy if one is found, and the best heuristic cost solution as described in \cite{Ho2024nhs} otherwise. 


We note that the comparison is limited to offline sampling-based methods without steering functions, excluding hierarchical TAMP (requirement of computationally expensive steering functions), MPC (local, not globally optimal), and online tree search (ineffective for long-horizon belief-space problems without advanced heuristics).

We implemented all algorithms in \texttt{C++} using OMPL \cite{sucan2012the-open-motion-planning-library}. All computations were performed single-threaded on a 3.60 GHz CPU with 32 GB RAM. For \poalg and \nhsalg, we used $k = 1000$ and $e = 0.05$ for all planning instances. All solvers are evaluated in an offline fashion.

\begin{figure*}[ht!]
    \centering 
    \begin{subfigure}{0.23\textwidth}
        \centering
        \includegraphics[width=\linewidth, height=0.14\textheight]{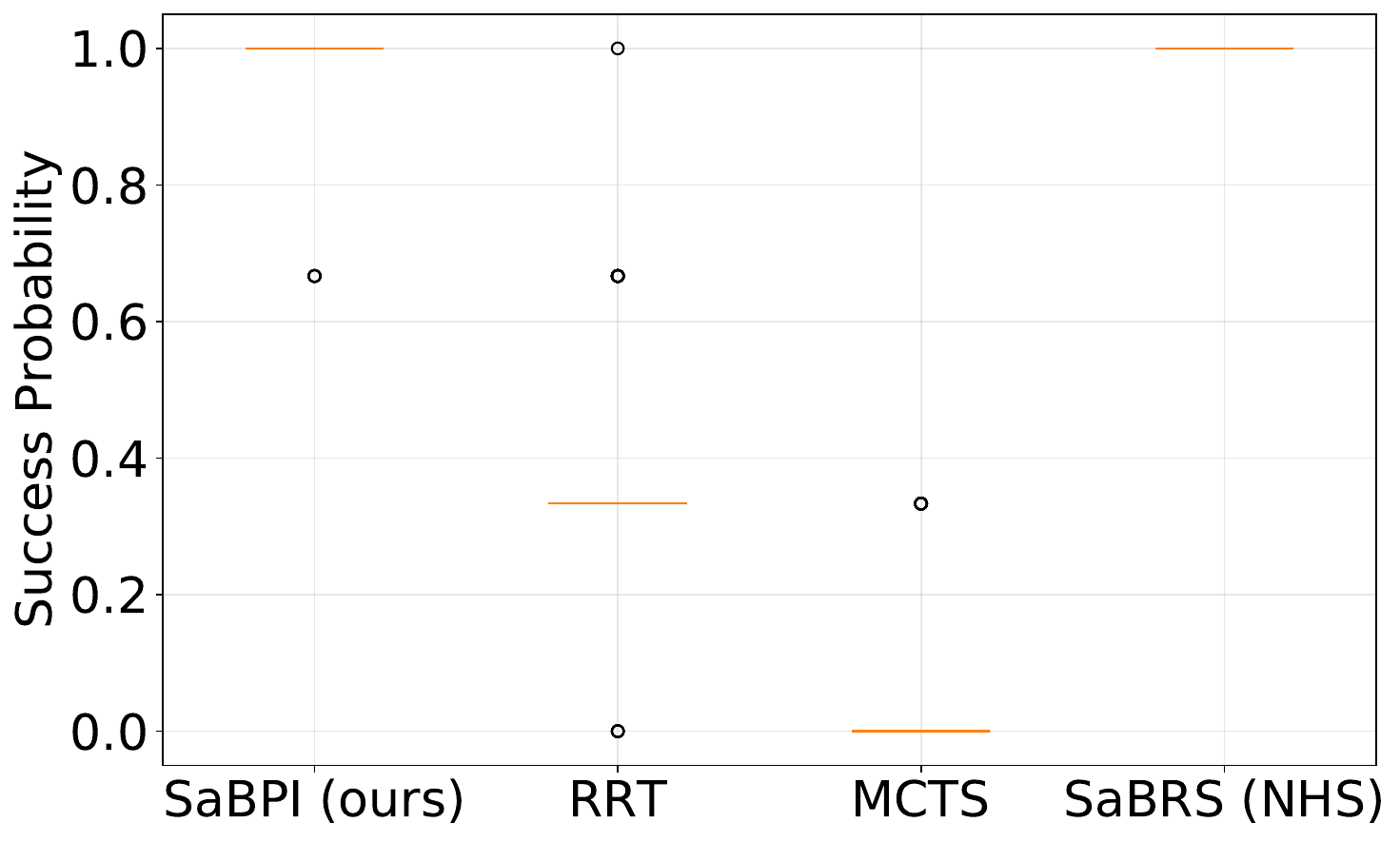}
        \caption{Door-Key}
        \label{fig:door-key}
    \end{subfigure}
    \begin{subfigure}{0.23\textwidth}
        \centering
        \includegraphics[width=\linewidth, height=0.14\textheight]{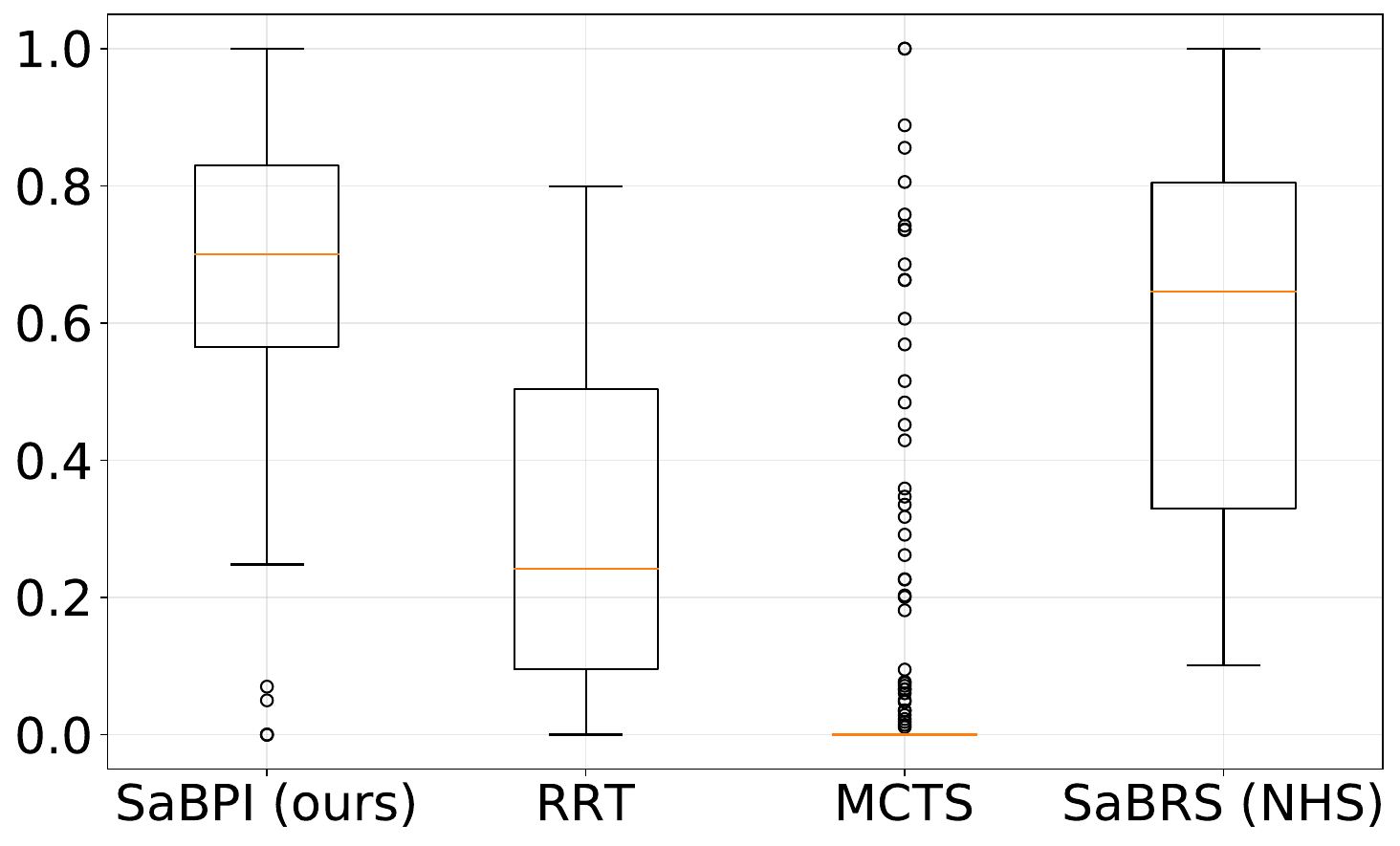}
        \caption{Fork}
        \label{fig:fork}
    \end{subfigure}
    \begin{subfigure}{0.23\textwidth}
        \centering
        \includegraphics[width=\linewidth, height=0.14\textheight]{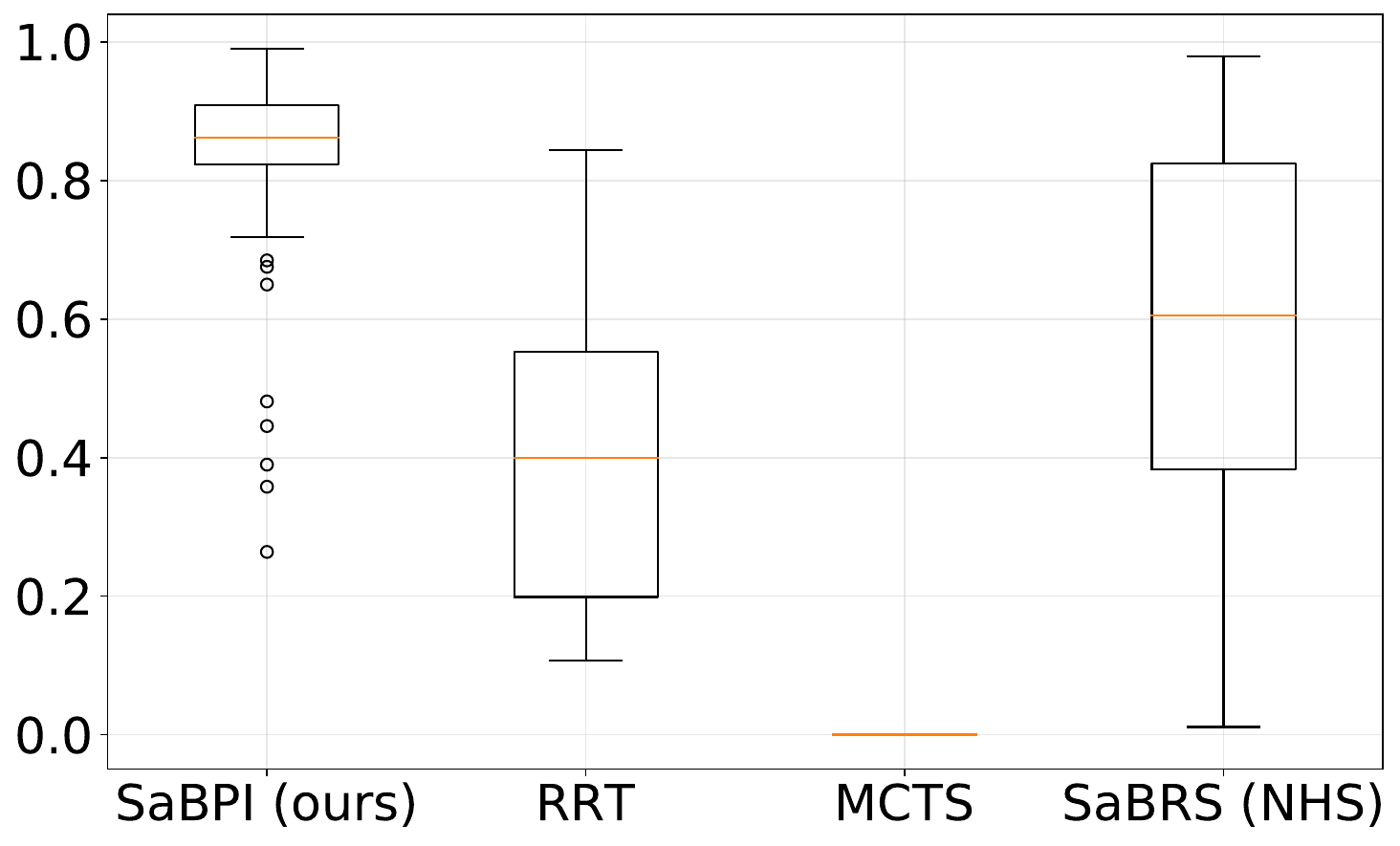}
        \caption{Continuous Rock-Sample}
        \label{fig:crs}
    \end{subfigure}
    \begin{subfigure}{0.23\textwidth}
        \centering
        \includegraphics[width=\linewidth, height=0.14\textheight]{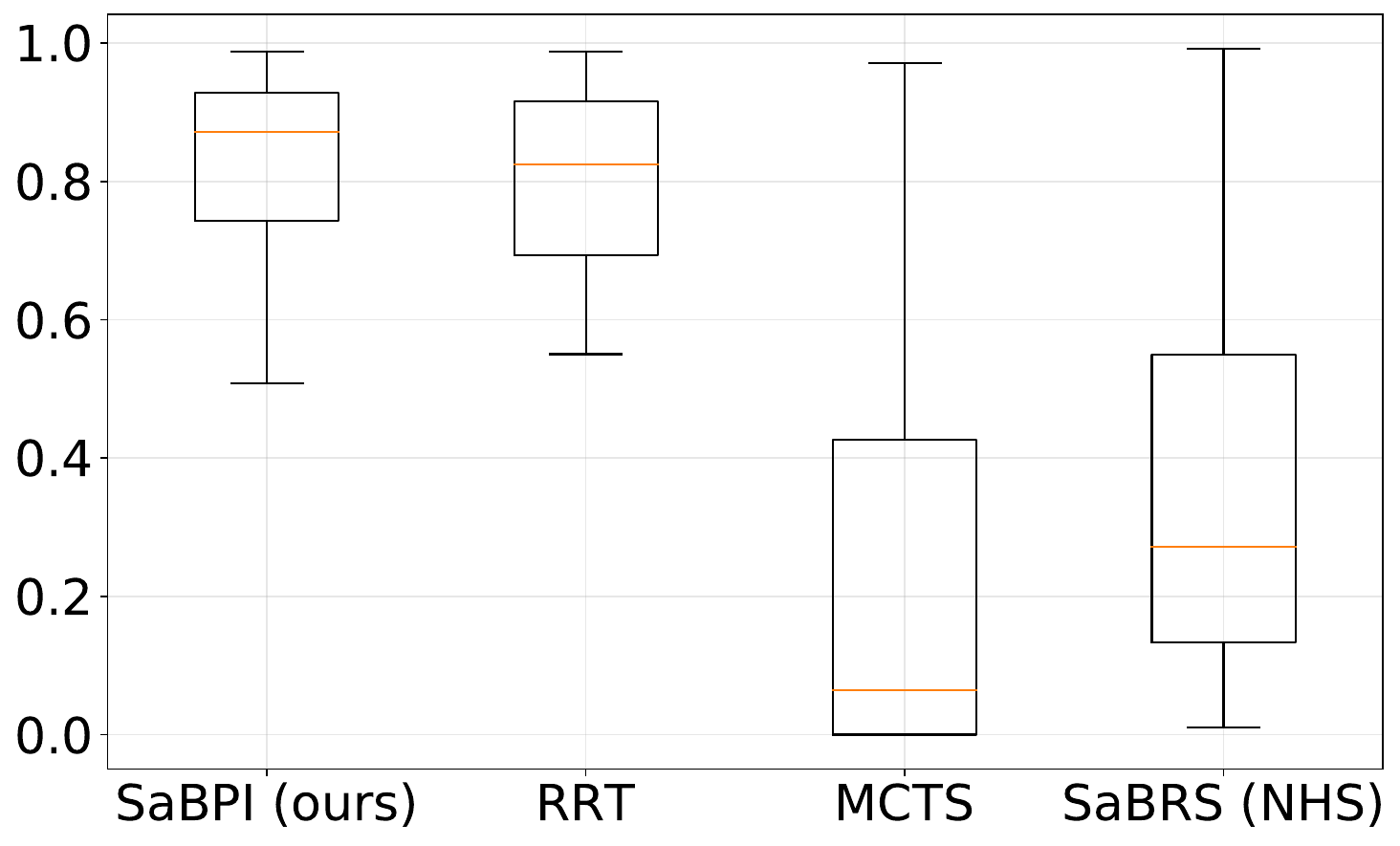}
        \caption{Fire Detection}
        \label{fig:firedetection}
    \end{subfigure}
    \caption{Benchmark Results of each algorithm for the four environments with a $60$ second time limit over $100$ trials.}
    \label{fig: benchmarks}
\end{figure*}

\begin{figure}[ht!]
    \centering
    \begin{subfigure}{0.5\columnwidth}
        \centering
        \includegraphics[width=\textwidth]{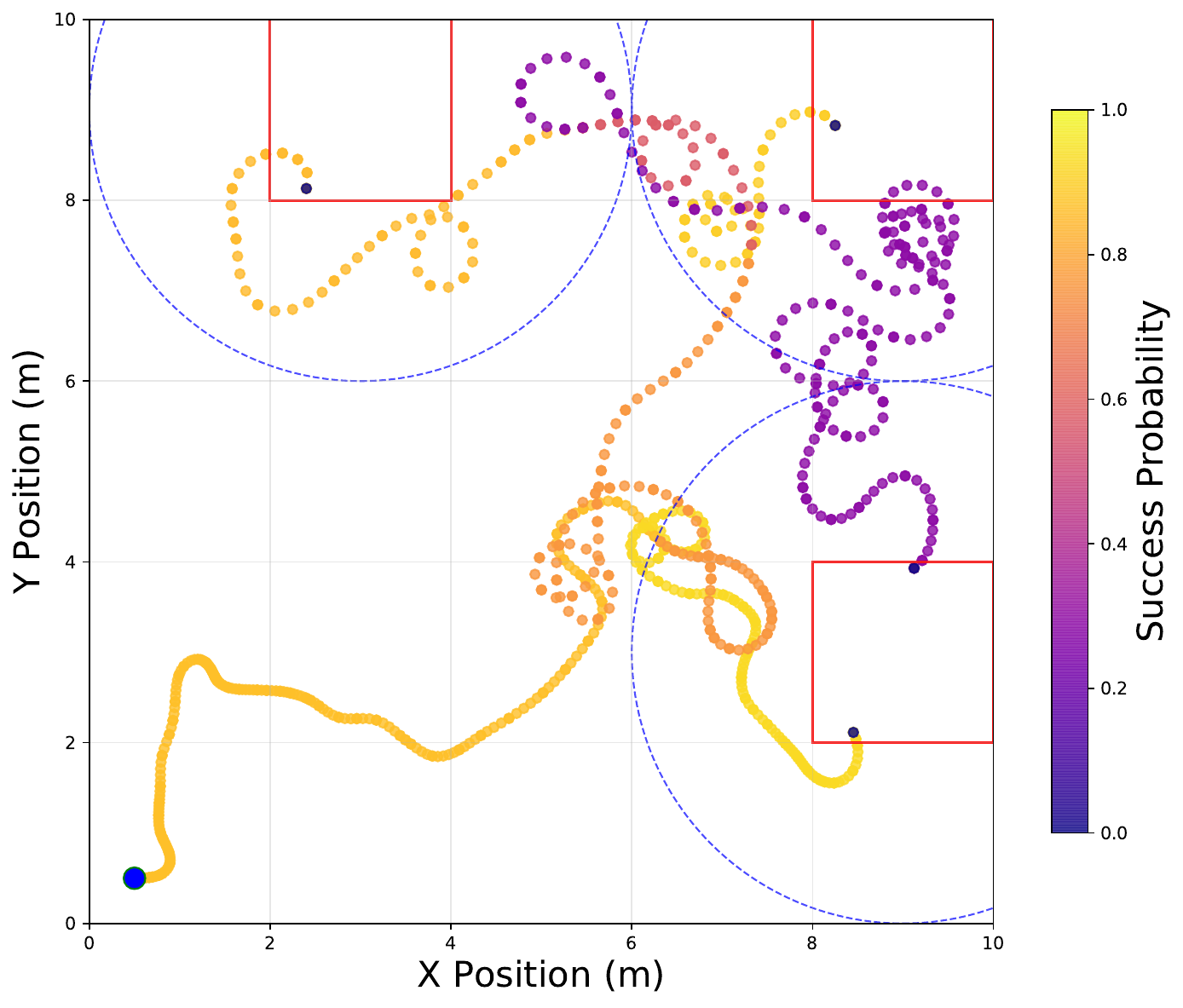}
        \caption{$b_0(good) = [0.5, 0.6, 0.7]$.}
        \label{fig: crs all good}
    \end{subfigure}%
    ~
    \begin{subfigure}{0.5\columnwidth}
        \centering
        \includegraphics[width=\linewidth]{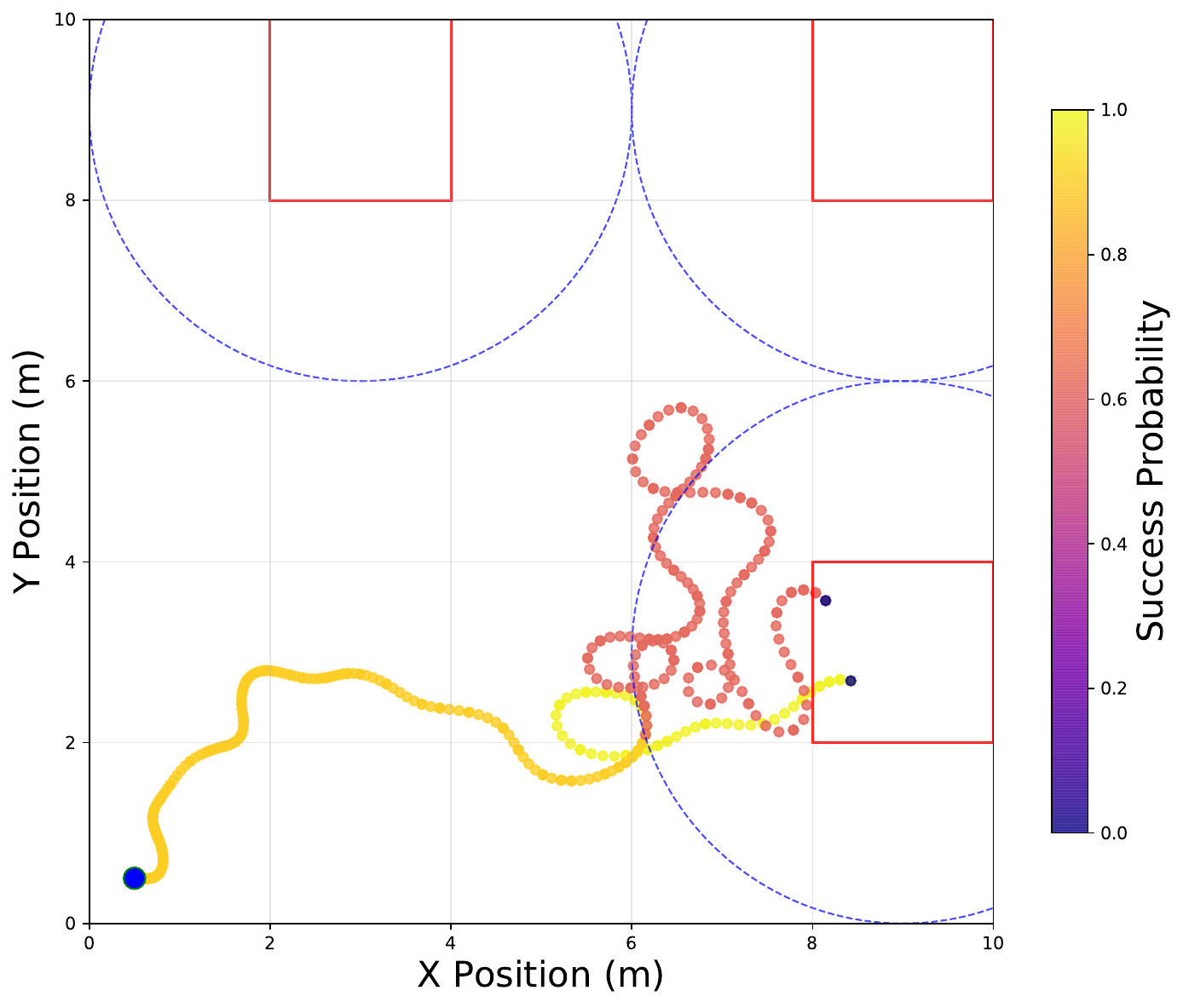}
        \caption{$b_0(good) = [0.2, 0.4, 0.9]$.}
        \label{fig: crs one good}
    \end{subfigure}
    \caption{C-Rock Sample with two different initial beliefs.}
\end{figure}

\begin{figure}[ht!]
    \centering
    \includegraphics[width=\linewidth, trim={7.0cm 4.0cm 7.0cm 4.0cm}, clip]{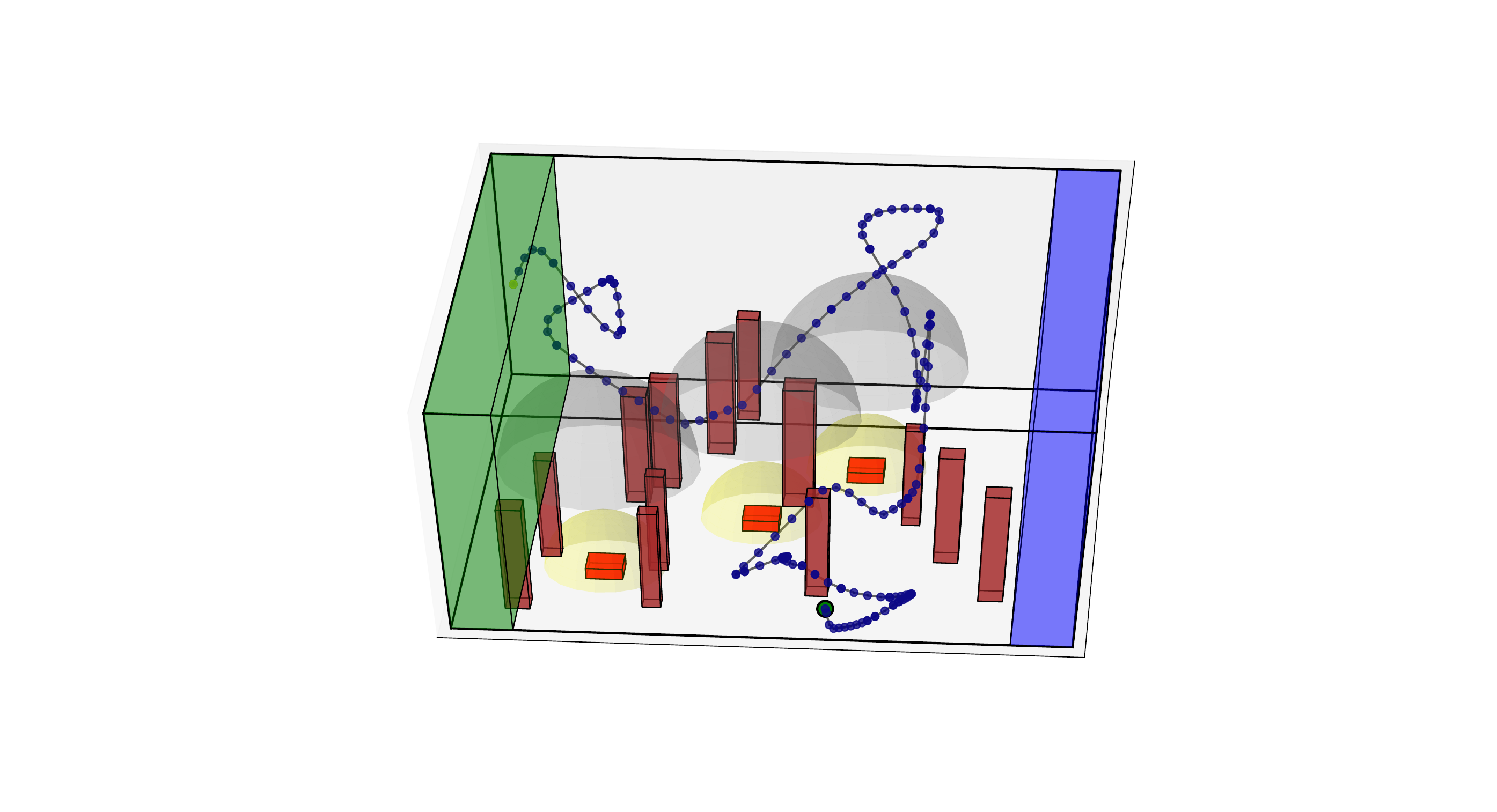}
    \caption{Drone Fire Detection Example: Exit A (green) and exit B (blue), fire sites (orange), obstacles (red). The yellow and gray hemispheres are observation regions.}
    \label{fig: drone fire}
\end{figure}

We considered the following instantiations of Problem~\ref{prob: policy}.

\begin{itemize}
    \item Door-Key: A robot with simple linear dynamics is tasked with finding the correct key in a set of regions before going to a door. The initial placement of the key is uncertain but observations are perfect. 
    The task is 
    \begin{align*}
        \phi = \eventually (correctkey) \wedge \eventually (door) \wedge \globally (door \rightarrow correctkey)
    \end{align*}
    \item Fork: A robot with second-order car dynamics is tasked with a reach-avoid objective to reach the exit while avoiding obstacles: $\phi = \neg obs U exit$. There are two potential paths (fork) to the exit, but there is uncertainty in whether it is traversable, which the robot can observe imperfectly when close to it.
    \item A continuous variant of rock sample (C-RS), Fig.~\ref{fig: crs one good}: A rover with second-order car dynamics and fuel constraints is tasked to collect a single rock and maximize the probability of it being a good rock. The task is 
    \begin{align*}
        \phi = \globally (fuel) \wedge \eventually (sample \implies good)
    \end{align*}\normalsize
    There is uncertainty in rock quality, and the robot can imperfectly observe its quality in a radius around it.
    \item Fire detection problem in Example~\ref{ex: drone}, Fig.~\ref{fig: drone fire}: 3D quadcopter dynamics tasked with detecting and reporting fire at potential sites. If there is a fire, it should exit to the left to report it, and if there is no fire, it should go to the right exit. Observations of smoke (less accurate observations) and site (more accurate observations) are imperfect. In $\ltlf$, the task is 
    \begin{align*}
        \globally (fire \rightarrow F(A)) \wedge \globally (\neg fire \rightarrow F(B)) \wedge \globally (\neg obs)
    \end{align*}
\end{itemize}

The optimal success probability is dependent on the belief and observation model. For the benchmarks, in all the environments, we randomize the initial belief distributions over the truth values of the atomic propositions in the uncertain regions. We ran each algorithm in each environment $100$ times, with a time limit of $60s$. 

\textbf{Benchmark Results: } Figure~\ref{fig: benchmarks} summarizes the probability of success of the computed policies. As seen from the results, \poalg performs the best among the compared algorithms. While each of the baselines occasionally exhibit strong performance in specific instances, they lack the consistency required in general for the variety of PO-SHS problems. \poalg provides a unified approach, combining the strategic policy selection of bandit-based stochastic decision-making with the efficient state-space exploration of RRT. This integration enables the rapid discovery of kinodynamically feasible paths while simultaneously optimizing for global task satisfaction in the hybrid belief space.

In the Door-Key example, in which there exist policies that can guarantee completion of the task with probability $1$, the worst-case semantics of \nhsalg enables it to find optimal solutions effectively. However, since \nhsalg does not use belief information, when it is necessary, especially in the other problems of Fork, C-RS, and Fire Detection, \nhsalg struggles to find policies with high probability of task satisfaction. When belief probability information is available, our framework enables to take advantage of that information and obtain the highest probability of task satisfaction.

From the performance of RRT and MCTS-PW, it is clear that neither a dominantly motion planning-based method nor a dominantly MCTS-based method is effective for finding for $\ltlf$ task satisfaction under propositional uncertainty. RRT is a motion planning algorithm that is able to search the state space efficiently. However, it does not have a good mechanism for optimizing for \emph{policies} that are conditioned on observations, causing conditional branching of the tree. Therefore, while it is able to find single trajectories with good probability of task satisfaction, it does not efficiently converge to good policies. MCTS-PW is a decision-making algorithm that optimizes for a policy tree directly. MCTS-PW does not perform well as it is not able to search the state space effectively without the presence of an efficient value estimate. However, since our problem is a temporal reachability problem with sparse rewards, MCTS-PW struggles to find good policies. Our algorithm \poalg uses combined decision-making selection of promising policies and the effective search of a state space with RRT in a single integrated algorithm, enabling fast exploration of the hybrid belief space while efficiently optimizing for task satisfaction.

\textbf{Illustrative Solutions: } We present illustrative solutions to demonstrate the effectiveness of the framework.

\textit{Continuous Rock Sample: } Figures~\ref{fig: crs all good} and \ref{fig: crs one good} show two different belief initializations of the quality of the rocks in the space. The robot starts in the bottom left (blue dot). The large blue circles are when the robot's sensors are close enough to observe, with $0.8$ probability, the quality of the nearest rock. in Fig.~\ref{fig: crs all good}, the robot plans a motion policy that visits all the rocks depending on the observations. In Fig.~\ref{fig: crs one good}, when there is one rock that has a very high initial belief of its quality, the robot decides to go directly to the sample that rock. The use of belief and observation modeling in our framework during planning enables \poalg to effectively optimize for adaptive policies with guarantees on satisfaction probabilities.

\textit{Fire Detection: } We show an example solution trajectory during execution of the Fire Detection problem in Fig.~\ref{fig: drone fire}, in which the computed policy has success probability of $0.93$. Initially, the robot has belief $b(fire) \approx [0.35, 0.84, 0.89]$ at the three potential sites. The robot plans a motion that goes to two of the sites to sense and updates its belief, and from both sites, it observes that there is a high probability of there being a fire. The robot then goes to exit A.

%% file: sections/Conclusion.tex
This paper considers the problem of linear temporal logic motion planning with uncertain environment semantics. We model the problem as a partially observable stochastic hybrid system, and propose an anytime, sound, and asymptotically optimal algorithm that maximizes task satisfaction probability. Empirical results show that the integration of strategic policy selection and SBMP enables efficient solutions with reachability guarantees. For future work, we plan to extend \poalg to handle uncertainty in continuous dynamics, and optimize both task satisfaction probability and another cost function.